\documentclass[conference]{IEEEtran}
\IEEEoverridecommandlockouts

\usepackage[utf8]{inputenc}
\usepackage[english]{babel}
\usepackage[T1]{fontenc}
\usepackage{aecompl}
\usepackage{cite}
\usepackage{amsmath,amssymb,amsfonts}
\usepackage[ruled,vlined]{algorithm2e}

\usepackage{afterpage}
\usepackage{graphicx}
\usepackage{textcomp}
\usepackage{xcolor}
\usepackage{amsthm}
\usepackage{paralist}
\usepackage{mathtools, nccmath}
\usepackage{setspace}
\usepackage{makecell}
\usepackage[inline]{enumitem}
\usepackage{subcaption}
\usepackage{bbm}
\usepackage{mathtools}
\usepackage{dsfont}

\theoremstyle{definition}

\newtheorem{theorem}{Theorem}
\newtheoremstyle{exampstyle}
{1pt} 
{1pt} 
{} 
{} 
{\bfseries} 
{} 
{.5em} 
{} 

\theoremstyle{exampstyle} 
\theoremstyle{exampstyle} \newtheorem{remark}{Remark}
\theoremstyle{exampstyle} \newtheorem{definition}{Definition}
\theoremstyle{exampstyle} \newtheorem{lemma}{Lemma}
\theoremstyle{exampstyle} 

\newcommand{\minus}{\scalebox{0.75}[1.0]{$-$}}

\usepackage{xpatch}
\makeatletter
\xpatchcmd{\@thm}{\thm@headpunct{.}}{\thm@headpunct{}}{}{}
\makeatother

\def\BibTeX{{\rm B\kern-.05em{\sc i\kern-.025em b}\kern-.08em
		T\kern-.1667em\lower.7ex\hbox{E}\kern-.125emX}}
\begin{document}
	
	\title{Binary Federated Learning with Client-Level Differential Privacy
		\thanks{This work is supported in part by the Hong Kong Research Grant Council under Grant No. 16208921 and NSFC/RGC Collaborative Research Scheme (CRS\_HKUST603/22).}
	}
	
	\author{
		\IEEEauthorblockN{Lumin Liu, Jun Zhang, \emph{Fellow, IEEE}, Shenghui Song, \emph{Senior Member, IEEE}, and Khaled B. Letaief, \emph{Fellow, IEEE}}\\
		\IEEEauthorblockA{Dept. of ECE, The Hong Kong University of Science and Technology, Hong Kong\\
			Email:{ lliubb@ust.hk,
				eejzhang@ust.hk,
				eeshsong@ust.hk,
				eekhaled@ust.hk}}
	}               
	
	\maketitle
	\begin{abstract}
		Federated learning (FL) is a privacy-preserving collaborative learning framework, and differential privacy can be applied to further enhance its privacy protection. Existing FL systems typically adopt Federated Average (FedAvg) as the training algorithm and implement differential privacy with a Gaussian mechanism. However, the inherent privacy-utility trade-off in these systems severely degrades the training performance if a tight privacy budget is enforced. Besides, the Gaussian mechanism requires model weights to be of high-precision. To improve communication efficiency and achieve a better privacy-utility trade-off, we propose a communication-efficient FL training algorithm with differential privacy guarantee. Specifically, we propose to adopt binary neural networks (BNNs) and introduce discrete noise in the FL setting.  Binary model parameters are uploaded for higher communication efficiency and discrete noise is added to achieve the client-level differential privacy protection. The achieved performance guarantee is rigorously proved, and it is shown to depend on the level of discrete noise.  Experimental results based on MNIST and Fashion-MNIST datasets will demonstrate that the proposed training algorithm achieves client-level privacy protection with performance gain while enjoying the benefits of low communication overhead from binary model updates.
	\end{abstract}
	
	\section{Introduction}\label{intro}
	Federated Learning (FL) has recently attracted considerable attention due to its ability to collaboratively and effectively train machine learning models while leaving the local private data untouched \cite{FLsurvey}.  Federated Average (FedAvg) \cite{mcmahan2017communication} is a popular FL training algorithm, which aggregates models trained by different clients via weight averaging and has been successfully implemented on real-world applications \cite{hard2018federated,ramaswamy2019federated}. However, subsequent studies revealed that privacy leakage still happens when malicious attackers obtain information about model weights, i.e., gradient inversion attack \cite{Hatamizadeh_2022_CVPR}, membership inference attack\cite{MIA2017}.\\
	
	Differential Privacy (DP) \cite{dwork2014algorithmic} is the \textit{de facto} metric for characterizing private data analysis due to its complete mathematical form and the information theoretical bound. The protection of the participants' privacy in FL is further enhanced with the Gaussian mechanism \cite{abadi2016deep, geyer2017differentially}, where a zero mean gaussian noise is added to the model weights before transmitting them to the parameter server, as shown in Fig. \ref{fig1}. Specifically, for the privacy amplification of Federated Edge Learning (FEEL), the analog transmission scheme is mostly adopted for its spectrum efficiency and the inherent channel noise, which could be treated as the privacy noise \cite{FL-DP, DP-OTA}. However, the Gaussian mechanism requires the model weights to be of high precision, which prohibits quantization techniques to further reduce the communication cost during training. Besides, to achieve satisfactory privacy protection, a significant amount of DP noise is needed and its adverse impact of model accuracy is inevitable, i.e., it leads to a privacy-utility trade-off. \\
	
	The enormous communication overhead of model uploading is another core challenge in FL, especially in wireless edge systems. Quantizing model updates improves the communication efficiency and presents a communication-accuracy trade-off in FL\cite{HierQSGD}. Quantization after adding Gaussian noise was investigated in \cite{quantizationeffectKim}. In \cite{cpsgd}, it was  proposed to first quantize the model updates and add discrete noise to improve the communication efficiency. It was shown in \cite{jin2020stochastic} that the stochastic-sign-sgd algorithm is differentially private. In these previous works, quantization, model optimization, and privacy are separately considered. In this way, the quantization step may degrade the model performance. This motivates us to consider models inherently with low-precision parameters to improve the communication efficiency for differentially private FL. \\
	
	To improve the communication efficiency, we propose to adopt the binary neural network (BNN). BNN was proposed in \cite{NIPS2016_BNN}, where it was shown that BNNs achieve nearly the same performance as  networks with full-precision model parameters. The training of BNNs requires a set of auxiliary full-precision parameters to track the gradients and update the binary model parameters accordingly. In \cite{FL-BNN}, the authors proposed a communication-efficient FL framework for training BNNs with uploading binary weights.\\
	
	\begin{figure}[t] 
		\centerline{\includegraphics[width=.99\linewidth]{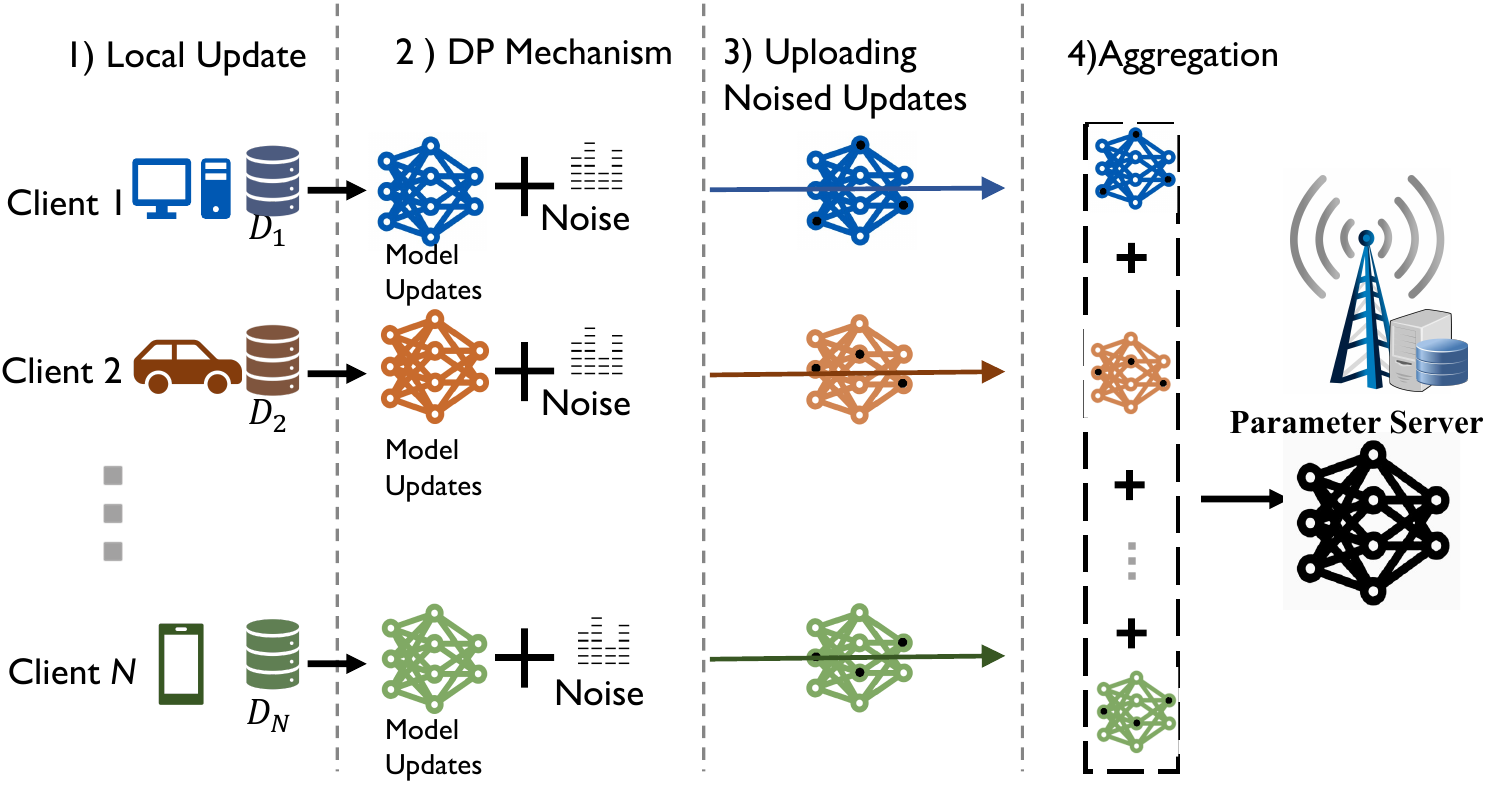}} 
		\caption{Illustration of a differentially-private FL system: Clients collaborate with the server to train a machine learning model. The uploaded model weights are protected with the DP mechanism by adding additional noises.}
		\label{fig1}
	\end{figure}
	
	In this paper, to improve privacy-utility trade-off and communication efficiency, we propose a communication-efficient and differentially private FL framework with BNNs, with Randomize Response as the noising mechanism. Specifically, client-level differential privacy is considered. We show that BNNs achieves privacy protection without degrading the model utility. Extensive experiments are provided to verify our design. 
	\\
	
	The paper is organized as follows. In Section \ref{sys_model}, we introduce the basics for FL system, DP, and BNNs. In Section \ref{algorithm}, we introduce the threat model and present the proposed algorithm. The experimental evaluation is given in Section \ref{simulations}, and Section \ref{conclusion} concludes the paper and provides discussions for possible future work. \\
	
	\section{System Model and Preliminary} \label{sys_model}
	In this section, we will introduce the FL system and the widely-adopted FedAvg algorithm. The preliminary knowledge about DP is then introduced and discussions on client-level DP in FL are provided. Last, the training process of BNN is explained.\\
	
	\subsection{Federated Learning}
	In FL, there are $n$ clients with local private datasets $\{ \mathcal{D}_i \}_{i=1}^n$ following the probability distribution $\{\mathcal{P}_i\}_{i=1}^n$. The dataset size of the $i$-th client is $D_i$. Based on the local dataset $\{ \mathcal{D}_i \}$, the empirical local loss function for the $i$-th client is expressed as
	\begin{equation}
		L_i (\mathbf{w}) = \frac{1}{D_i} \sum_{ {\{ \mathbf{x}_j, y_j \} } \in \mathcal{D}_i}\mathcal{L}(\mathbf{w},\mathbf{x}_j, y_j) \label{eq:4},
	\end{equation}
	where $\mathcal{L}(\mathbf{w},\mathbf{x}_j, y_j)$ is the loss function of the training data sample $\mathbf{x}_j$ with label $y_j$, and $\mathbf{w}$ denotes the model parameters. 
	The target in FL is to learn a global model that performs well on the average of the local data distributions. Denote the joint dataset as $\mathcal{D} = \bigcup_{i=1}^n \mathcal{D}_i $ then the target training loss function in FL is given by
	\begin{equation}
		L (\mathbf{w}) = \frac{1}{\sum_{i=1}^{n}D_i} \sum_{{\xi_j}\in \mathcal{D}}\mathcal{L}(\mathbf{w},\xi_j) = \frac{1}{\sum_{j=1}^{n}D_j} \sum_{i=1}^n D_i L_i(\mathbf{w}). \label{eq:5}
	\end{equation}
	The most commonly adopted training algorithm in FL is FedAvg, where each client periodically updates its model locally and averages the local model parameters through communications with a central server (e.g., at the cloud or edge).
	The parameters of the local model on the $i$-th client after $t$ steps of stochastic gradient descent (SGD) iterations are denoted as $\mathbf{w}_t^i$. In this case,  $\mathbf{w}_t^i$ evolves as follows
	
	\begin{equation}
		\text{$\mathbf{w}_t^i$ } = 
		\begin{cases}
			\text{$\mathbf{w}_{t-1}^i - \eta  \tilde{\nabla} L_i(\mathbf{w}_{t-1}^i)$} &  \text{$t \mid \tau \neq 0$}\\
			\text{ $\frac{1}{n} \sum_{i=1}^n[\mathbf{w}_{t-1}^i - \eta \tilde{\nabla} L_i(\mathbf{w}_{t-1}^i)]$ } &
			\text{$t \mid \tau = 0$}
		\end{cases} \label{eq:6}
	\end{equation}
	

	\subsection{Differential Privacy}
	DP is a rigorous privacy metric for measuring the privacy risk by computing the distribution divergence between the outcomes of two neighboring datasets. Approximate DP, or $(\epsilon, \delta)$-DP is the most classic notion of DP \cite{dwork2014algorithmic}, defined as follows:\\
	
	\begin{definition} {($(\epsilon, \delta)$-DP).}
		An algorithm $\mathcal{M}: \mathcal{X}^n\rightarrow \mathcal{Y}$ is $(\epsilon, \delta)$-DP if, for all neighboring databases $X, X' \in \mathcal{X}^n$ and all $T \subseteq \mathcal{Y}$,
		\begin{equation*}
			Pr[M(X)\in T] \leq e^\epsilon Pr[M(X')\in T] +\delta.
		\end{equation*}
		where the neighboring databases $X, X'$ only differ in one data point, and $\epsilon \geq 0$ and $0<\delta<1$.
	\end{definition}
	
	A relaxed notion of DP is later proposed in \cite{rdp}.
	
	\begin{definition}[($\alpha, \rho(\alpha)$-RDP)]
		An algorithm $\mathcal{M}: \mathcal{X}^n\rightarrow \mathcal{Y}$ is $(\alpha, \rho(\alpha))$-RDP if, for all neighboring databases $X, X' \in \mathcal{X}^n$, the \textit{Rényi $\alpha$-divergence} between $M(X)$ and $M(X')$ satisfies:
		\begin{equation*}
			\begin{split}
				\mathcal{D}_\alpha(M(X) \| M(X')) \triangleq & \frac{1}{\alpha-1} \log \mathrm{E}\left[\frac{M(X)}{M(X')}\right]^\alpha \\
				\leq & \rho(\alpha).
			\end{split}
		\end{equation*}
		where $\alpha>1$ and $\rho \geq0$.
	\end{definition}
	Since in FL, multiple rounds of communication, i.e., queries, are necessary, the composition theorem is needed to compute the privacy loss over multiple rounds. RDP has a tighter composition bound and thus is more suitable to analyze the end-to-end privacy loss of an iterative algorithm. The conversion from RDP to $(\epsilon, \delta)$-DP is as follows: \\
	
	\begin{lemma}[RDP to $(\epsilon, \delta)$-DP]
		If an algorithm $\mathcal{M}$ satisfies $(\alpha, \rho(\alpha))$-RDP, then it also satisfies $(\rho(\alpha)+\frac{log(1/\delta)}{\alpha-1}, \delta)$-DP.
		\label{lemma1}
	\end{lemma}
	\vspace{2mm}
	\begin{lemma}[Composition Theorem for RDP\cite{rdp}]
		For randomized mechanisms $\mathcal{M}_1$ and $\mathcal{M_2}$ applied on dataset $X$, if $\mathcal{M}_1$ satisfies $(\alpha, \rho_1(\alpha))$-RDP, $\mathcal{M}_2$ satisfies $(\alpha, \rho_2(\alpha))$-RDP, then their composition $\mathcal{M_1}\circ\mathcal{M_2}$ satisfies $(\alpha, \rho_1+\rho_2)$-RDP.
		\label{lemma2}
	\end{lemma}
	For BNNs, there are only binary values of the model weights, adding continuous noise, i.e., gaussian noise or laplace noise, would increase the communication overhead. Here we adopt discrete noise, i.e., Randomize Response, as the noising mechanism. 
	\begin{definition}[Randomize Response]
		For a query function $h: X \rightarrow \{0, 1\}$, the Randomize Response mechanism $\mathbf{RR}(\cdot)$ for $h$ is defined as :
		\begin{equation}
			\mathbf{RR}_\gamma(h(D)) = 
			\begin{cases}
				h(D) &  \text{with probability } 1/2+\gamma,\\
				1-h(D) & \text{with probability } 1/2-\gamma.
			\end{cases} 
		\end{equation}, where $0< \gamma <1/2$.
	\end{definition}
	\vspace{2mm}
	\begin{lemma}[RDP for Random Response\cite{rdp}]
		The Randomize Response mechanism $\mathbf{RR}_\gamma$ satisfies:
		\small
		\begin{equation*}
			\left( 
			\alpha, 
			\frac{log
				\left(
				(\frac{1}{2}+\gamma)^\alpha(\frac{1}{2}-\gamma)^{(1-\alpha)}
				+(\frac{1}{2}-\gamma)^\alpha (\frac{1}{2}+\gamma)^{1-\alpha}
				\right)}{\alpha-1}
			\right)
		\end{equation*}
		-RDP, if $\alpha>1$, and 
		\begin{equation*}
			\left(\alpha, 2\gamma log\frac{1/2+\gamma}{1/2-\gamma}
			\right)-\text{RDP}
		\end{equation*} if $\alpha=1$
	\label{lemma3}
	\end{lemma}
	\begin{remark}
		The privacy loss, $\rho(\alpha)$,  increases with the increase of $\gamma$, with $\gamma=0$ meaning no privacy loss, and $\gamma=1/2$ meaning no privacy protection. Fig. \ref{fig2} is provided to illustrate the monotonicity.
	\end{remark}
	\begin{figure}[t] 
		\centerline{\includegraphics[width=.9\linewidth]{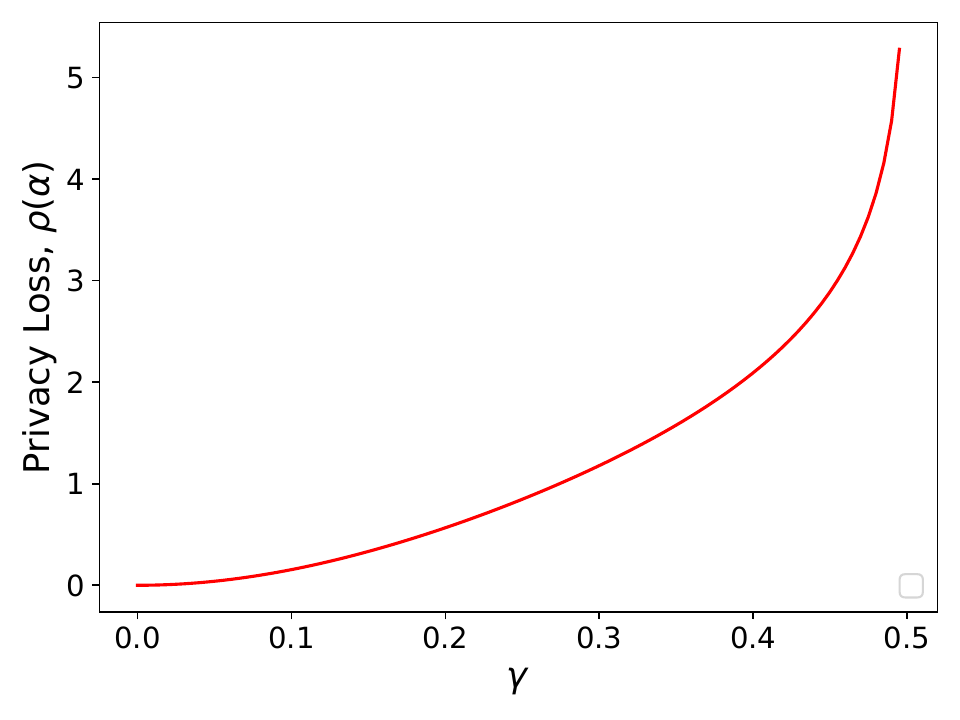}} 
		\caption{Illustration of the privacy loss of a random response $\mathbf{RR}_\gamma(\cdot)$, we set $\alpha = 2$ in this figure.}
		\label{fig2}
	\end{figure}
	\subsection{Client-Level DP in FL}
	
	In the context of FL, the model updates of different participating clients are transmitted independently, necessitating protection from the untrusted cloud server. Given the distributed nature of FL, in this work we consider client-level DP. Specifically, two datasets $D$ and $D'$ are neighboring datasets if $D \cup \{D_c\}$  or $D \setminus \{D_c\}$ is identical to $D'$ for a client $c$, where $D_c$ denotes all the data points associated with client $c$. This definition ensures that the privacy guarantee holds for all data points belonging to that client, and is stronger than the commonly-used notion of sample-level DP, which only protects the addition or removal of a single data point for a client. It is deemed more suitable for FL settings with large numbers of clients and each client holding a small dataset.
	
	To enhance privacy, secure aggregation (SecAgg) is a widely-used technique in the literature for enabling client-level DP in FL \cite{SecAgg}. This technique is a lightweight form of cryptographic secure multi-party computation that prevents the server from inspecting individual model updates of clients in FL. By allowing the server to learn only an aggregate function of the clients’ local model updates, typically the sum, SecAgg improves privacy. In this paper, we follow the approach taken in existing FL works\cite{LowRank}. We treat the SecAgg as a black-box and ignore the finite precision and modular summation arithmetic associated with secure aggregation in this paper.\\
	\subsection{Binary Neural Networks (BNNs)}
	BNN is a specialized type of neural networks that was initially introduced in \cite{NIPS2016_BNN}. It is characterized by highly compressed parameters, with its neural network weights being binary in nature. As a result, BNNs require significantly less storage space for model inference compared with classical neural networks. The compressed parameters also bring a low communication overhead when adopting BNN in FL. Compared with other post-quantization training methods, BNN is a quantization-aware training method and thus achieves better utility with a less model weight precision. The training process of BNNs is as follows.
	
	Consider an L-layer binary neural network, where the weights and activations are constrained to +1 or -1. We denote the weights (including trainable parameters of the activation) of each layer by $W_\ell^b \in \{1, \minus 1\}^*$, where $\ell = 1, ..., L$ and $*$ represents the dimension of $W_\ell^b$. The output of layer $\ell$ is given by:
	\begin{equation*}
		\mathbf{a}_\ell = f_\ell(W_\ell^b, \mathbf{a}_{\ell-1})
	\end{equation*}
	where $\mathbf{a_{\ell-1}}$ is the output of layer $\ell-1$ and $f_\ell(\cdot)$ denotes the operation of layer $\ell$ on the input.
	
	Training BNN with SGD requires auxiliary full-precision parameters, which is denoted as $\bar{W}$ and is of the same size as $W^b$. SGD algorithms proceed at training iteration $t$ with a learning rate $\eta$ as follows, where $t= 1, \dots, T$:
	\begin{equation}
		\bar{W}_t = \bar{W}_{t-1} - \eta \frac{\partial f(W_{t-1}^b, \mathcal{D}_t)}{\partial W_{t-1}^b}
		\label{eq:sgd}
	\end{equation}
	with $\mathbf{D}_t$ denoting the training data batch at iteration $t$, and $f(\cdot)$ is the loss function. $W_t^b$ is updated as:
	\begin{equation}
		W_t^b = Sign(W_t),
		\label{eq:binarize}
	\end{equation}
	where $Sign(\cdot)$ is an element-wise operation which returns the sign of each element.
	The values of the auxiliary parameters are restricted within in $(-1, +1)$, as in \cite{FL-BNN, NIPS2016_BNN}. Hence, at the end of each iteration, we have:
	\begin{equation*}
		\bar{W}_t = clip(\bar{W}_t, -1,1)=
		\begin{cases}
			-1,& \quad \bar{W}_t <-1 \\
			\bar{W}_t&, \quad -1\leq \bar{W}_t \leq 1\\
			1, &\quad \bar{W}_t >1.
		\end{cases}
	\end{equation*}\\
	
	\section{Proposed DPFL-BNN Algorithm} \label{algorithm}
	In this section, we present the DPFL-BNN training algorithm, which is shown to achieve a better privacy-utility trade-off and high communication efficiency. We first describe the threat model considered in this work, and we propose the DPFL-BNN algorithm and analyze its privacy and convergence guarantee.
	
	\subsection{Threat Model}
	The threat model considers adversaries from an ``honest but curious" server or clients in the system. The adversary honestly adheres to the training protocol but is curious about a client's training data and attempts to infer it from the information exchanged during the training process. Besides, the adversary could also be a passive outside attacker, who eavesdrops on the exchanged information during training process but will not actively attack the system, i.e., injecting false message or interrupting message transmissions. \\
	

	\begin{algorithm}[t] 
		\setstretch{1}
		\SetAlgoLined
		Initialize local model $\{\mathbf{W}^i\} = W_0^b$ and server model $\mathbf{W^s} = W_0^b$ \\
		\For{t = 0,1,\dots, T}{
			\{ Start a global iteration. \}\\
			\{ 1. Local training phase \} \\
			\For{client $i \in \mathcal{C}$}{
				\{ All clients train a BNN simutaneously\}\\
				Calculate the loss of BNN: $L_i = f(W_{t-1}^{b,i}, \mathcal{D}_i)$\\
				Update local parameters $\bar{W_{t}^i}$ with eq. \eqref{eq:sgd}\\
				Binarize: $W_t^{b,i} = StoSign(\bar{W}_t^i)$ \\
				Clip: $\bar{W_{t}^i} = clip(\bar{W_{t}^i}, -1, 1)$\\
			}
			\{2. Noise adding process\} \\
			Each client adding noise to the binarized model weights with $\mathbf{RR}_\gamma$:\\
			$\tilde{W}_t^{b,i} = \mathbf{RR}_\gamma(W_t^{b,i})$\\
			\{3. Parameter uploading process\}\\
			Each client encryt $\tilde{W}_t^{b,i}$ and send it to the parameter server via secure aggregation.\\
			\{4. Server aggregating process\}\\
			Server decrypts to get
			$\tilde{W}_t= \sum_{i=1}^N W_t^{b,i}$ \\
			\{5.Parameter downloading process\}\\
			Each clients downloads $\tilde{W}_t$ from server, and update local auxiliary parameters $W_t^i$ and binary parameters $W_t^{i,b}$ as:\\
			$W_t^i = \beta \tilde{W}_t + (1-\beta) \bar{W_{t}^i}$,\\
			$W_t^{i,b} = Sign(W_t^i).$\\
		}
		\caption{DPFL-BNN}
		\label{algorithm2}
	\end{algorithm}
	
	\subsection{DPFL-BNN}
	Our proposed DPFL-BNN algorithm is developed based on the FedAvg algorithm, with several key steps modified to improve the communication efficiency and protect data privacy. 
	Particularly, each client trains a BNN model locally. The binary model parameters are noised before being uploaded to the parameter server. To maintain the communication efficiency brought by the binary model parameters, discrete noise is added, i.e. the Random Response mechanism is adopted. The noised parameters are then sent to the server via SecAgg and the server decrypts the averaged weights without knowing individual updates. The averaged weights are then sent back to clients for next round of training. In the following discussion, we will treat the SecAgg as a black box which will faithfully compute the aggregated sum without further revealing private information since the there are many works on SecAgg and discussions on this are beyond the scope of this paper. 
	The entire training process of DPFL-BNN is illustrated in Algorithm \ref{algorithm2}, and the training details are described as follows. \\
	
	At the beginning of training round $t$, each client maintains two sets of model parameters, i.e., binary parameters $W_{t-1}^{b,i}$, and auxiliary full-precision parameters $W_{t-1}^i$. Then local update is performed on local dataset $\mathcal{D}_i$ with eq. \eqref{eq:sgd}, \eqref{eq:binarize} to update $W_{t}^i$ and $W_{t}^{b,i}$. \\
	
	The straightforward way to aggregate information in a private way from clients is to aggregate the auxiliary model parameters $W_t^i$ with an appropriate amount of Gaussian noise added. This, however, incurs a high communication cost since $W_t^i$ is full-precision. To reduce the communication cost, we only upload the binary parameters as in \cite{FL-BNN}. To protect privacy without incurring more communication overhead, discrete noise is added to $W_t^{b,i}$ instead of Gaussian noise. The noised model parameters are $\tilde{W}_t^{b,i} = \mathbf{RR}_\gamma(W_t^{b,i})$.\\
	
	The server aggregates $\tilde{W}_t^{b,i}$ and update $\tilde{W}_t= \sum_{i=1}^N W_t^{b,i}$. $\tilde{W}_t$ is then disseminated to clients to update $W_t^i$ and $W_t^{i,b}$ for the next round of training, where $W_t^i = \beta \tilde{W}_t + (1-\beta) \bar{W_{t}^i}$,
	$W_t^{i,b} = Sign(W_t^i).$ , where $\beta \in [0,1] $ is a hyper-parameter controlling the significance of the aggregated information from the server. \\
	
	Now, we show that the proposed algorithm satisfies agent-level DP.
	
	\subsection{Privacy Analysis}
	The level of DP guarantee of the training algorithm depends on the values of the fliping probability $\gamma$. According to Lemma \ref{lemma3}, the privacy loss increases with $\gamma$, with $\gamma=0$ meaning no privacy loss, and $\gamma=1/2$ meaning no privacy protection. Different from the Gaussian noise mechanism, we do not need to control the sensitivity since all the weights are already binarized. We can then use RDP to account the total privacy loss across $T$ rounds. The final DP guarantee can be obtained by converting RDP back to DP, as shown in Theorem 1.
	\begin{theorem}[Privacy Guarantee of DPFL-BNN]
		Assume that the updates are noised with $\mathbf{RR}_\gamma(\cdot)$ during each communication round. After $T$ rounds of training, DPFL-BNN satisfies $(\epsilon, \delta)$-DP for any $\delta \in (0,1)$ if 
		\begin{equation}
			\gamma \leq g^{-1}(\frac{1/\delta}{T}),
		\end{equation}
	where $g^{-1}(\cdot)$ is the inverse function of 
	\begin{equation*}
		\begin{split}
			g(\gamma) = & log(
			(\frac{1}{2}+\gamma)^{1+\frac{2log(1/\delta)}{\epsilon}}(\frac{1}{2}-\gamma)^{-\frac{2log(1/\delta)}{\epsilon}} \\
			+&(\frac{1}{2}-\gamma)^{1+\frac{2log(1/\delta)}{\epsilon}} (\frac{1}{2}+\gamma)^{-\frac{2log(1/\delta)}{\epsilon}} ).
		\end{split}
	\end{equation*}
	\end{theorem}

	\begin{proof}
		By Lemma \ref{lemma3}, each round of training satisfies $\left( 
		\alpha, \rho(\alpha)
		\right)$-RDP, with 
		
		\begin{equation*}
			\begin{split}
				\rho(\alpha) =\frac{1}{\alpha-1}log
				(
				(&\frac{1}{2}+\gamma)^\alpha(\frac{1}{2}-\gamma)^{(1-\alpha)}\\
				+&(\frac{1}{2}-\gamma)^\alpha (\frac{1}{2}+\gamma)^{1-\alpha}
				),
			\end{split}
		\end{equation*}
	Then by Lemma \ref{lemma2}, the algorithm satisfies $\left(\alpha, T\rho(\alpha) \right)$ after $T$ rounds of training. To guarantee $(\epsilon, \delta)$-DP, according to Lemma \ref{lemma1}, we need:
		\begin{equation*}
			T\rho(\alpha)+\frac{log(1/\delta)}{\epsilon} \leq \epsilon,
		\end{equation*}
	Let $\alpha = 1+2\frac{1/\delta}{\epsilon}$, and notice that $\rho(\alpha)$ is monotonically increasing with $\gamma$ and takes values between $(0, \infty)$ then we have:
	$
		\gamma \leq g^{-1}(\frac{1/\delta}{T}).
	$
	\end{proof}

\vspace{-3mm}
	\section{Simulations} \label{simulations}
	In this section, we present simulation results to demonstrate the performance of the proposed DPFL-BNN algorithm.
	\subsection{Experimental Settings}
	\subsubsection{Datasets and Models}
	We conduct experiments on MNIST \cite{mnist} and Fashion-MNIST \cite{fashionmnist}, a common benchmark for differentially private machine learning. The MNIST dataset has been shown to demonstrate good utility with strong privacy guarantee and is also the dataset considered in \cite{FL-BNN} for BNNs in FL.
	While the Fashion-MNIST is considered as “solved” in the computer vision community, achieving high utility with strong privacy guarantee remains difficult on this dataset. The Fashion-MNIST dataset consists of 60,000 28 × 28 grayscale images of 10 fashion categories, along with a set of 10,000 test samples. \\
	
	To simulate the independent and identically distributed (IID) data distribution, we randomly split the training data among 100 agents.
	To simulate different degrees of non-IID data splits, we utilize the Dirichlet distribution $Dir(\alpha)$ as in \cite{FDactive} with a larger $\alpha$ indicating a more homogeneous data distribution. Particularly in our experiment, $\alpha=100$ represents the IID case, $\alpha=1$ represents the non-IID case.\\
	
	We use a CNN model for both datasets, which consists of two 3 × 3 convolution layers (both with 16 filters, each followed with 2 × 2 max pooling), two fully connnected layers with 784 and 100 units, and a final softmax output layer. Note that all the layers use $tanh(\cdot)$ as the activation function due to the BNN training strategies.
	\subsubsection{Hyperparameter Settings}
	For all experiments, we simulate an FL setting with 100 clients, number of communication rounds $T = 100$, and a local update step of 10. We use Adam with momentum as the local optimizer of agents with a momentum coefficient of 0.5 and batch size of 64. The local learning rate is initialized as 0.1 and decayed at a rate of 0.1 every 40 communication rounds. The hyperparameter for BNN training, $\beta$ is set as 0.3, which follows the optimal setting in \cite{FL-BNN}. The test accuracy is averaged across all clients since each client holds different local models. All experiments are conducted with the same hyperparameter settings for three times with different random seeds of pytorch and python.\footnote{The implementation code could be found at github : https://github.com/LuminLiu/BinaryFL} 
	
	\subsection{Simulation Results}
	\begin{figure}
		\centering
		\begin{subfigure}[b]{0.43\textwidth}
			\centering
			\includegraphics[width=\textwidth]{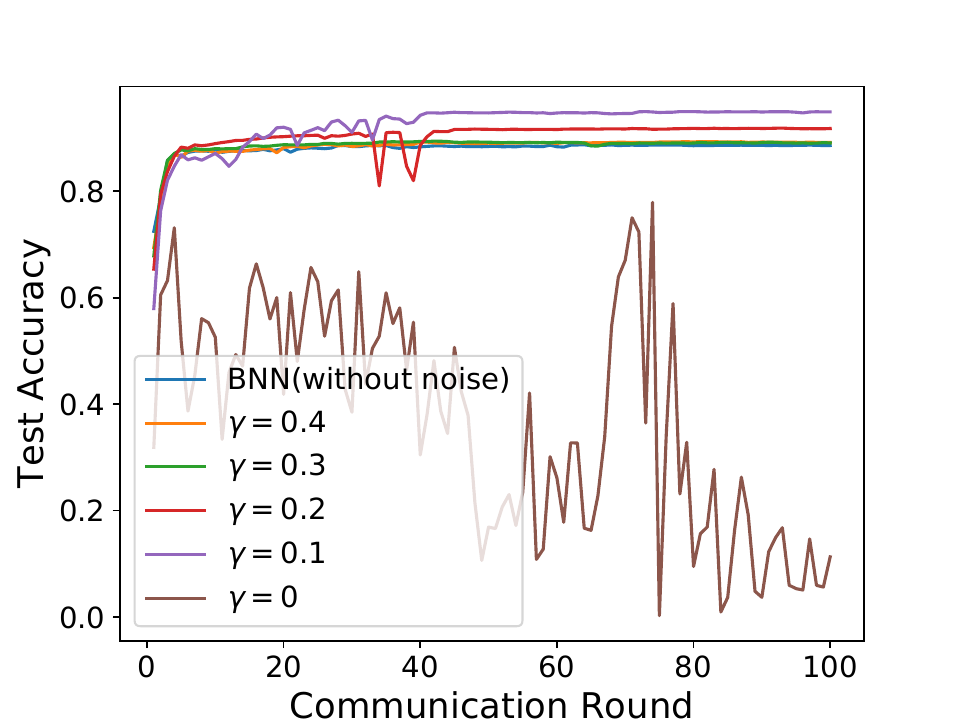}
			\caption{IID data distribution.}
			\label{fig:mnistiid}
		\end{subfigure}
		\hfil
		\begin{subfigure}[b]{0.43\textwidth}
			\centering
			\includegraphics[width=\textwidth]{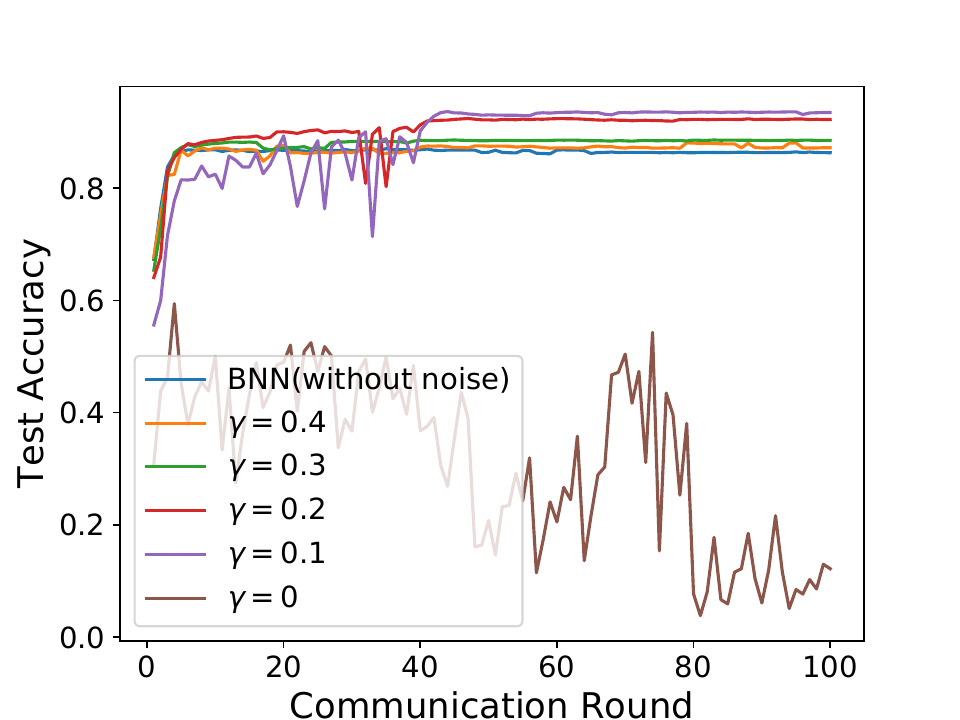}
			\caption{non-IID data distribution.}
			\label{fig:mnistnoniid}
		\end{subfigure}
		\caption{MNIST, 100 clients.}
		\label{fig:mnist}
	\end{figure}
	
		\begin{figure}
		\centering
		\begin{subfigure}[b]{0.43\textwidth}
			\centering
			\includegraphics[width=\textwidth]{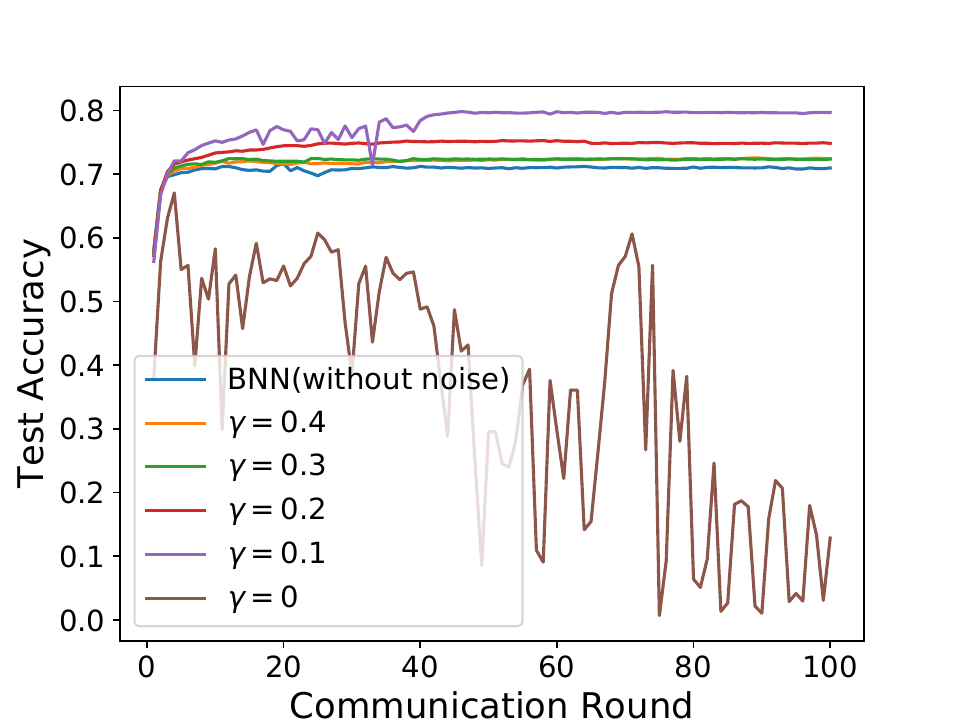}
			\caption{ IID data distribution.}
			\label{fig:fashionmnistiid}
		\end{subfigure}
		
		\begin{subfigure}[b]{0.43\textwidth}
			\centering
			\includegraphics[width=\textwidth]{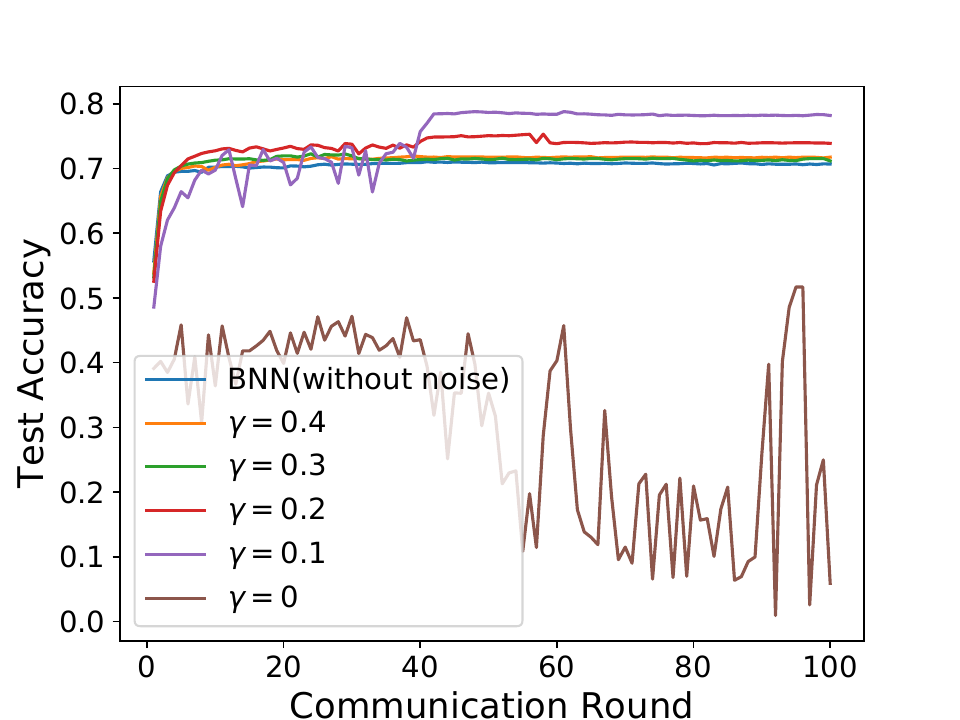}
			\caption{ non-IID data distribution.}
			\label{fig:fashionmnistnoniid}
		\end{subfigure}
	\caption{Fashion-MNIST, 100 clients.}
	\label{fig:fashionmnist}
	\end{figure}

	We present the simulation results for MNIST in Fig. \ref{fig:mnist} and Fashion-MNIST in Fig. \ref{fig:fashionmnist}. 
	BNN is the baseline with no privacy guarantee. As we decrease the value of $\gamma$ from 0.4 to 0.1, which means better privacy is achieved. Also a better training performance is achieved. $\gamma = 0.1$ presents the best accuracy, which outperforms the baseline without noise by a large margin. However, setting $\gamma$ as 0, which means full privacy and complete random response to the query, will degrade the performance.
	
	\section{Conclusions} \label{conclusion}
	FL is a new paradigm that has been attracting significant attention from industry and academia due to its ability to protect clients' privacy of raw data. In this paper, we have proposed DPFL-BNN, a novel differentially private FL scheme via binary neural networks to achieve client-level DP with high model accuracy in FL.  Experimental results have demonstrated the superior performance of our approach with privacy guarantee. In future work, we will investigate this unusual privacy-utility trade-off in details.

	\bibliographystyle{IEEEtran}  
	\bibliography{ref.bib}  
	
\end{document}